\documentclass[runningheads]{llncs}
\usepackage{graphicx}
\usepackage{comment}
\usepackage{amsmath,amssymb} 
\usepackage{color}

\usepackage{epsfig}
\usepackage{graphicx}
\usepackage{mathtools}

\usepackage{multirow}
\usepackage{xcolor}
\usepackage{siunitx} 
\usepackage{booktabs} 
\usepackage{soul}

\usepackage[linesnumbered,ruled,vlined]{algorithm2e}
\SetKwInput{KwInput}{Input}                
\SetKwInput{KwOutput}{Output}              
\DeclareMathOperator*{\argmin}{arg\,min}
\DeclarePairedDelimiter\abs{\lvert}{\rvert}
\newtheorem{lemm}{Lemma}

\usepackage{pifont}

\begin{document}
\pagestyle{headings}
\mainmatter

\title{Post-Training Piecewise Linear Quantization \\ for Deep Neural Networks} 

\titlerunning{Post-Training Piecewise Linear Quantization for Deep Neural Networks}
%

\author{Jun Fang\inst{1} \and
Ali Shafiee\inst{1} \and
Hamzah Abdel-Aziz\inst{1} \and
David Thorsley\inst{1} \and \\ 
Georgios Georgiadis\inst{2}\thanks{Work done at Samsung Semiconductor, Inc.} \and
Joseph Hassoun\inst{1}
}

\authorrunning{J. Fang et al.}
%

\institute{Samsung Semiconductor, Inc. \\
\email{\{jun.fang, ali.shafiee, hamzah.a, d.thorsley, j.hassoun\}@samsung.com} 
\and Microsoft \quad  \email{gegeo@microsoft.com}
}

\maketitle

\begin{abstract}
Quantization plays an important role in the energy-efficient deployment of deep neural networks on resource-limited devices. Post-training quantization is highly desirable since it does not require retraining or access to the full training dataset. The well-established uniform scheme for post-training quantization achieves satisfactory results by converting neural networks from full-precision to 8-bit fixed-point integers. However, it suffers from significant performance degradation when quantizing to lower bit-widths.
In this paper, we propose a \textbf{p}iece\textbf{w}ise \textbf{l}inear \textbf{q}uantization (PWLQ) scheme to enable accurate approximation for tensor values that have bell-shaped distributions with long tails. Our approach breaks the entire quantization range into non-overlapping regions for each tensor, with each region being assigned an equal number of quantization levels. Optimal breakpoints that divide the entire range are found by minimizing the quantization error. Compared to state-of-the-art post-training quantization methods, experimental results show that our proposed method achieves superior performance on image classification, semantic segmentation, and object detection with minor overhead. 
\keywords{deep neural networks, post-training quantization, piecewise linear quantization}
\end{abstract}

\section{Introduction}\label{S:intro}

In recent years, deep neural networks (DNNs) have achieved state-of-the-art results in a variety of learning tasks including image classification \cite{hu2018squeeze,huang2017densely,szegedy2016rethinking,he2016deep,simonyan2014very,krizhevsky2012imagenet}, segmentation \cite{chen2018encoder,he2017mask,ronneberger2015u} and detection \cite{liu2016ssd,redmon2017yolo9000,ren2015faster}. 
Scaling up DNNs by one or all of the dimensions \cite{tan2019efficientnet} of network depth \cite{he2016deep}, width \cite{zagoruyko2016wide} or image resolution \cite{lai2018fast} attains better accuracy, at a cost of higher computational complexity and increased memory requirements, which makes the deployment of these networks on embedded devices with limited resources impractical.

One feasible way to deploy DNNs on embedded systems is quantization of full-precision (32-bit floating-point, FP32) weights and activations to lower precision (such as 8-bit fixed-point, INT8) integers \cite{jacob2018quantization}. By decreasing the bit-width, the number of discrete values is reduced, while the quantization error, which generally correlates with model performance degradation increases. To minimize the quantization error and maintain the performance of a full-precision model, many recent studies \cite{zhou2017incremental,cai2017deep,micikevicius2017mixed,jacob2018quantization,choi2018pact,zhang2018lq,faraone2018syq,jung2019learning} rely on training either from scratch (``quantization-aware" training) or by fine-tuning a pre-trained FP32 model.

However, post-training quantization is highly desirable since it does not require retraining or access to the full training dataset. It saves time-consuming fine-tuning effort, protects data privacy, and allows for easy and fast deployment of DNN applications. 
Among various post-training quantization schemes proposed in the literature \cite{krishnamoorthi2018quantizing,choukroun2019low,zhao2019improving}, uniform quantization is the most popular approach to quantize weights and activations since it discretizes the domain of values to evenly-spaced low-precision integers which can be efficiently implemented on commodity hardware's integer-arithmetic units.

Recent work \cite{krishnamoorthi2018quantizing,lee2018quantization,nagel2019data} shows that post-training quantization based on a uniform scheme with INT8 is sufficient to preserve near original FP32 pre-trained model performance for a wide variety of DNNs. However, ubiquitous usage of DNNs in resource-constrained settings requires even lower bit-width to achieve higher energy efficiency and smaller models. 
In lower bit-width scenarios, such as 4-bit, post-training uniform quantization causes significant accuracy drop \cite{krishnamoorthi2018quantizing,zhao2019improving}. This is mainly because the distributions of weights and activations of pre-trained DNNs is bell-shaped such as Gaussian or Laplacian \cite{han2015deep,lin2016fixed}. 
That is, most of the weights are clustered around zero while few of them are spread in a long tail. As a result, when operating at low bit-widths, uniform quantization assigns too few quantization levels to small magnitudes and too many to large ones, which leads to significant accuracy degradation \cite{krishnamoorthi2018quantizing,zhao2019improving}.

To mitigate this issue, various quantization schemes
\cite{miyashita2016convolutional,cai2017deep,baskin2018uniq,park2018value,jain2019biscaled,li2020additive} are designed to take advantage of the fact that weights and activations of pre-trained DNNs typically have bell-shaped distributions with long tails. Here, we present a new number representation via a piecewise linear approximation to be suited for these phenomena. It breaks the entire quantization range into \textit{non-overlapping regions} where each region is assigned an equal number of quantization levels. 
Although our method works with an arbitrary number of regions, we suggest limiting them to two to simplify the complexity of the proposed approach and the hardware overhead.
The \textit{optimal breakpoints} that divide the entire range can be found by minimizing the quantization error. Compared to uniform quantization, our piecewise linear quantization (PWLQ) provides a richer representation that reduces the quantization error. This indicates its potential to reduce the gap between floating-point and low-bit precision models. It is also more hardware-friendly when compared to other non-linear approaches such as logarithm-based and clustering-based approaches~\cite{miyashita2016convolutional,ullrich2017soft,baskin2018uniq}, since in our method, computation can still be carried out without the need of any transforms or look-up tables.

The main contributions of our work are as follows:
\begin{itemize} 
    \item[$\bullet$] We propose a piecewise linear quantization (PWLQ) scheme for efficient deployment of pre-trained DNNs without retraining or access to the full training dataset. We also investigate its impact on hardware implementation.
   
    \item[$\bullet$] {We present a solution to find the optimal breakpoints and demonstrate that our method achieves a lower quantization error than the uniform scheme.} 
    
    \item[$\bullet$] We provide a comprehensive evaluation on image classification, semantic segmentation, and object detection benchmarks and show that our proposed method achieves state-of-the-art results. 

\end{itemize}
\section{Related Work}\label{S:relatedwork}

There is a wide variety of approaches in the literature that facilitate the efficient deployment of DNNs. The first group of techniques relies on designing network architectures that depend on more efficient building blocks.
Notable examples include depth/point-wise layers \cite{howard2017mobilenets,sandler2018mobilenetv2} as well as group convolutions \cite{zhang2018shufflenet,ma2018shufflenet}. These methods require domain knowledge, training from scratch and full access to the task datasets. 
The second group of approaches optimizes network architectures in a typical task-agnostic fashion and may or may not require (re)training. Weight pruning \cite{han2015deep,li2016pruning,he2017channel,luo2017thinet}, activation compression \cite{dong2017more,dhillon2018stochastic,georgiadis2019accelerating}, knowledge distillation \cite{hinton2015distilling,polino2018model} and quantization \cite{courbariaux2015binaryconnect,rastegari2016xnor,zhu2016trained,zhou2016dorefa,miyashita2016convolutional,jacob2018quantization} fall under this category.

In particular, quantization of activations and weights \cite{gong2014compressing,gupta2015deep,wu2016quantized,lin2016fixed,choi2018pact,zhang2018lq,zhao2019improving} leads to model compression and acceleration as well as to overall savings in power consumption. Model parameters can be stored in a fewer number of bits while the computation can be executed on integer-arithmetic units rather than on power-hungry floating-point ones~\cite{jacob2018quantization}.
There has been extensive research on quantization with and without (re)training. In the rest of this section, we focus on post-training quantization that directly converts full-precision pre-trained models to their low-precision counterparts.

Recent works~\cite{krishnamoorthi2018quantizing,lee2018quantization,nagel2019data} have demonstrated that 8-bit quantized models have been able to accomplish negligible accuracy loss for a variety of networks. To improve accuracy, per-channel (or channel-wise) quantization is introduced in \cite{krishnamoorthi2018quantizing,lee2018quantization} to address variations of the range of weight values across channels. Weight equalization/factorization is applied by \cite{meller2019same,nagel2019data} to rescale the difference of weight ranges between different layers. In addition, bias shifts in the mean and variance of quantized values are observed and counteracting methods are suggested by~\cite{banner2018post,finkelstein2019fighting}. A comprehensive evaluation of clipping techniques is presented by \cite{zhao2019improving} along with an outlier channel splitting method to improve quantization performance. Moreover, adaptive processes of assigning different bit-width for each layer are proposed in \cite{lin2016fixed,zhou2018adaptive} to optimize the overall bit allocation.

There are also a few attempts to tackle 4-bit post-training quantization by combining multiple techniques. In \cite{banner2018post}, a combination of analytical clipping, bit allocation, and bias correction is used, while \cite{choukroun2019low} minimizes the mean squared quantization error by representing one tensor with one or multiple 4-bit tensors as well as by optimizing the scaling factors.

Most of the aforementioned works utilize a linear or uniform quantization scheme. However, linear quantization cannot capture the bell-shaped distribution of weights and activations, which results in sub-optimal solutions. 
To overcome this deficiency, \cite{baskin2018uniq} proposes a quantile-based method to improve accuracy but their method works efficiently only on highly customized hardware; \cite{jain2019biscaled} employs two different scale factors on overlapping regions to reduce computation bits over fixed-point implementations. However, its scale factors restricted to powers of two and heuristic options limit the accuracy performance.
Instead, we propose a piecewise linear approach that improves over the selection of optimal breakpoints that leads to state-of-the-art quantized model results. Our method can be implemented efficiently with minimal modification to commodity hardware.
\section{Quantization Schemes}\label{S:q_scheme}

In this section, we review a uniform quantization scheme and discuss its limitations. We then present PWLQ, our piecewise linear quantization scheme and show that it has a stronger representational power (a smaller quantization error) compared to the uniform scheme. 
\begin{figure}
\begin{center}
   \includegraphics[width=0.95\linewidth]{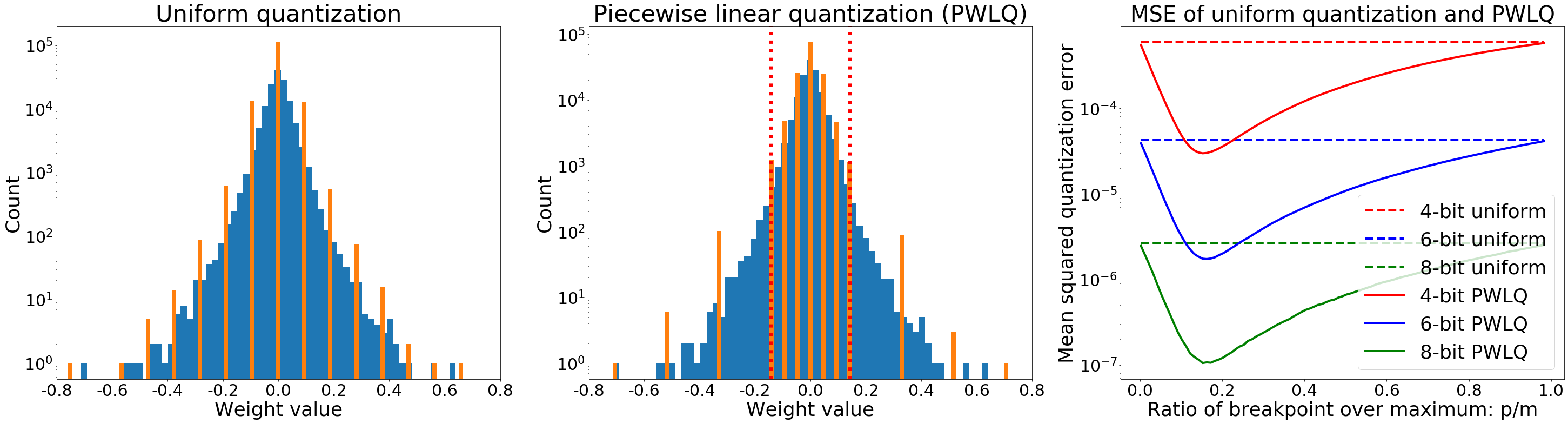}
\end{center}
   \caption{
   Quantization of $conv4$ layer weights in a pre-trained Inception-v3. Left: uniform quantization. Middle: piecewise linear quantization (PWLQ) with one breakpoint, dotted line indicates the breakpoint. Right: Mean squared quantization error (MSE) for various bit-widths ($b=4, 6, 8$). MSE of PWLQ is convex $w.r.t.$ the breakpoint $p$, the $b$-bit PWLQ can achieve a smaller quantization error than the $b$-bit uniform scheme
}
\vspace{-1mm}
\label{fig:pw-uniform-compare}
\end{figure}

\subsection{Uniform Quantization}

Uniform quantization (the left of Figure \ref{fig:pw-uniform-compare}) linearly maps full-precision real numbers $r$ into low-precision integer representations. 
From \cite{jacob2018quantization,choukroun2019low}, the approximated version $\hat{r}$ from uniform quantization scheme at $b$-bit can be defined as: 
\begin{equation} \label{eq:uniform-quant-scheme}
\begin{array}{ll}
&\hat{r} =  \text{uni} (r; b, r_{l}, r_{u}, z)  =  s \times r_q + z, \vspace{0.5mm}\\
& r_q = \Bigl\lceil \frac{\text{clamp} (r; r_{l}, r_{u}) - z}{s} \Bigr\rfloor_{\mathbb{Z}_b}, \vspace{0.5mm}\\
& \text{clamp}(r; r_{l}, r_{u}) = \min( \max(r, r_{u}), r_{l}), \vspace{0.5mm}\\ 
& s = \frac{\Delta}{N-1}, \quad \Delta = r_{u} - r_{l}, \quad N=2^b, 
\end{array}
\end{equation}
where $[r_l, r_u]$ is the quantization range, $s$ is the scaling factor, $z$ is the offset, $N$ is the number of quantization levels, $r_q$ is the quantized integer computed by a rounding function $\lceil \cdot \rfloor_{\mathbb{Z}_b}$ followed by saturation to the integer domain $\mathbb{Z}_b$. We set the offset $z=0$ for symmetric signed distributions combined with $\mathbb{Z}_b = \lbrace -2^{b-1}, ..., 2^{b-1} -1 \rbrace$ and $z=r_l$ for asymmetric unsigned distributions (e.g., ReLU-based activations) with $\mathbb{Z}_b = \lbrace 0, ..., 2^{b} -1 \rbrace$. 
Since the scheme \eqref{eq:uniform-quant-scheme} introduces a quantization error defined as $\varepsilon_{uni} = \hat{r} - r$, the expected quantization error squared is given by:
\begin{equation} \label{eq:uni-quant-err}
\mathbb{E} (\varepsilon_{uni}^2; b, r_{l}, r_{u}) = \frac{s^2}{12} = C(b) \Delta^2, 
\end{equation}
with $C(b) = \frac{1}{12 (2^b-1)^2}$ under uniform distributions~\cite{you2010audio}.

From the above definition, uniform quantization divides the range evenly despite the distribution of $r$. Empirically, the distributions of weights and activations of pre-trained DNNs are similar to bell-shaped Gaussian or Laplacian~\cite{han2015deep,lin2016fixed}. Therefore, uniform quantization is not always able to achieve small enough approximation error to maintain model accuracy, especially in low-bit cases.

\subsection{Piecewise Linear Quantization (PWLQ)}\label{S:PW_def}

To improve model accuracy for quantized models, we need to approximate the original model as accurately as possible by minimizing the quantization error. We follow this natural criterion to investigate the quantization performance, even though no direct relationship can easily be established between the quantization error and the final model accuracy~\cite{choukroun2019low}.

Inspired from \cite{park2018value,jain2019biscaled} that takes advantage of bell-shaped distributions, our approach based on piecewise linear quantization is designed to minimize the quantization error. It breaks the quantization range into two non-overlapping regions: the dense, central region and the sparse, high-magnitude region. An equal number of quantization levels $N=2^b$ is assigned to these two regions. We chose to use two regions with one breakpoint to maintain simplicity in the inference algorithm (Section \ref{sec:ablation-study}) and the hardware implementation (Section \ref{S:HW_impact}). Multiple-region cases are discussed in Section \ref{sec:ablation-study}.

Therefore, we only consider one breakpoint $p$ to divide the quantization range\footnote{Here we consider symmetric quantization range $[-m, m]$ ($m > 0$) for simplicity, it is extendable to asymmetric ranges $[m_1, m_2]$ for any real numbers $m_1 < m_2$.} $[-m, m]$ ($m > 0$) into two symmetric regions: the center region $R_1 = [-p,p]$ and the tail region $R_2 = [-m,-p) \cup (p,m]$. 
Each region consists of a negative piece and a positive piece. Within each of the four pieces, $(b-1)$-bit ($b \geq 2$) uniform quantization \eqref{eq:uniform-quant-scheme} is applied such that including the sign every value in the quantization range is being represented into $b$-bit. We define the $b$-bit piecewise linear quantization (denoted by PWLQ) scheme as:
\begin{equation} \label{eq:pw-scheme1}
\begin{array}{ll} 
\text{pw}(r; b, m, p)=\left\{
\begin{array}{ll}
     \mbox{sign}(r)  \times \text{uni}(\abs{r}; b-1, 0, p, 0), r \in R_1 \vspace{1mm}\\ 
     \mbox{sign}(r) \times \text{uni}(\abs{r}; b-1, p, m, p), r \in R_2 \\ 
\end{array} 
\right. ,
\end{array}
\end{equation}
where the sign of full-precision real number $r$ is denoted by $\mbox{sign}(r)$. The associated quantization error is defined as $\varepsilon_{pw} = \text{pw}(r; b, m, p) - r$.

Figure~\ref{fig:pw-uniform-compare} shows the comparison between uniform quantization and PWLQ on the empirical distribution of the $conv4$ layer weights in a pre-trained Inception-v3 model \cite{szegedy2016rethinking}. We emphasize that $b$-bit PWLQ represents FP32 values into $b$-bit integers to support $b$-bit multiply-accumulate operations, even though in total, it has the same number of quantization levels as $(b+1)$-bit uniform quantization. The implications of this are further discussed in Section \ref{S:HW_impact}.

\subsection{Error Analysis} \label{sec:error-analysis}

To study the quantization error for PWLQ, we suppose the full-precision real number $r$ has a symmetric probability density function (PDF) $f(r)$ on a bounded domain $[-m, m]$ with the cumulative distribution function (CDF) $F(r)$ satisfying $f(r) = f(-r)$ and $F(-m)=0$, $F(m)=1$. Then, we calculate the expected quantization error squared of PWLQ from \eqref{eq:uni-quant-err} based on the error of each piece:
\begin{equation} \label{eq:pw-variance}
\begin{array}{ll}
\mathbb{E}  (\varepsilon_{pw}^2; b, m, p) &= C(b-1) \Big\lbrace (m-p)^2  \big[F(-p) + 1 - F(p)\big] \\
& + \: p^2  [F(p) - F(-p)]   \Big\rbrace,
\end{array}
\end{equation}
Since $F(r) = 1 - F(-r)$ for a symmetric PDF, equation~\eqref{eq:pw-variance} can be simplified as:
\begin{equation} \label{eq:pw-variance2}
\begin{array}{ll}
\mathbb{E}  (\varepsilon_{pw}^2; b, m, p) = C(b-1) \Big\lbrace (m-p)^2  + m(2p-m) \big[2F(p) -1\big] \Big\rbrace.
\end{array}
\end{equation}

The performance of a quantized model with PWLQ scheme critically depends on the value of the breakpoint $p$. If $p=\frac{m}{2}$, then the PWLQ is essentially equivalent to uniform quantization, because the four pieces have equal quantization ranges and bit-widths. If $p<\frac{m}{2}$, the center region has a smaller range and greater precision than the tail region, as shown in the middle of Figure~\ref{fig:pw-uniform-compare}. Conversely, if $p>\frac{m}{2}$, the tail region has greater precision than the center region. To reduce the overall quantization error for bell-shaped distributions found in DNNs, we increase the precision in the center region and decrease it in the tail region. Thus, we limit the breakpoint to the range $0 < p < \frac{m}{2}$.

Accordingly, the optimal breakpoint $p^*$ can be estimated by minimizing the expected squared quantization error:
\begin{equation} \label{eq:p_optimize}
\begin{array}{ll}
p^* = \argmin_{p \in (0, \frac{m}{2})} {\mathbb{E}  (\varepsilon_{pw}^2; b, m, p)}.
\end{array}
\end{equation}
Since bell-shaped distributions tend to zero as $r$ becomes large, we consider a smooth $f(r)$ is decreasing when $r$ is positive, i.e., $f'(r) < 0$, $\forall r>0$. Then we prove that the optimization problem \eqref{eq:p_optimize} is convex with respect to the breakpoint $p \in (0, \frac{m}{2})$. Therefore one unique $p^*$ exists to minimize the quantization error \eqref{eq:pw-variance2}, as demonstrated by the following \textbf{Lemma~\ref{L:pw_convex}}.

\begin{lemm}\label{L:pw_convex}
If $f(-r) = f(r)$, $f'(r) < 0$ for all $r>0$, then $\mathbb{E}  (\varepsilon_{pw}^2; b, m, p)$ is a convex function of the breakpoint $p \in (0, \frac{m}{2})$.
\end{lemm}
\begin{proof}
Taking the first and second derivatives of \eqref{eq:pw-variance2} yields: 
\begin{equation} \label{eq:pw-variance-first-derivative}
\begin{array}{ll}
\frac{\partial  \mathbb{E} (\varepsilon_{pw}^2; b, m, p)}{\partial p} = 2C(b-1) \Big[  p-2m  + 2mF(p) + m(2p-m)f(p) \Big],
\end{array}
\end{equation}
\begin{equation} \label{eq:pw-variance-second-derivative}
\begin{array}{ll}
\frac{\partial^2  \mathbb{E} (\varepsilon_{pw}^2; b, m, p)}{\partial p^2}  = 2C(b-1) \Big[ 1 + 4mf(p)  + m(2p-m)f'(p) \Big],
\end{array}
\end{equation}
Since $f'(p)<0$ and $p < \frac{m}{2}$, $m(2p-m)f'(p) > 0$, then $\frac{\partial^2  \mathbb{E} (\varepsilon_{pw}^2; b, m, p)}{\partial p^2} > 0$. Therefore, $\mathbb{E} (\varepsilon_{pw}^2; b,m,p)$ is convex w.r.t. $p$, and thus a unique $p^*$ exists. 
\end{proof}

In practice, we can find the optimal breakpoint by solving \eqref{eq:p_optimize} by assuming an underlying Gaussian or Laplacian distribution using gradient descent \cite{rumelhart1986learning}.
Once the optimal breakpoint $p^*$ is found, both~\textbf{Lemma~\ref{L:pw_minimal}} and the numerical simulation in the right of Figure~\ref{fig:pw-uniform-compare} show that PWLQ achieves a smaller quantization error than uniform quantization, which indicates its stronger representational power.

\begin{lemm}\label{L:pw_minimal}
$\mathbb{E} (\varepsilon_{pw}^2; b, m, p^*) < \frac{C(b-1)}{16 C(b)} \mathbb{E} (\varepsilon_{uni}^2; b, -m, m)$ for $b \geq 2$. 
\end{lemm}
\begin{proof}
The $b$-bit uniform quantization error on $[-m, m]$ is calculated from \eqref{eq:uni-quant-err}:
\begin{equation} \label{eq:uni-quant-err-[-m,m]}
    \mathbb{E} (\varepsilon_{uni}^2; b, -m, m) =  C(b)(2m)^2 = 4 C(b) m^2 .
\end{equation}
For $b$-bit PWLQ, we solve the convex problem \eqref{eq:p_optimize} by letting the first derivative  equal to zero in \eqref{eq:pw-variance-first-derivative}, and determine that the optimal breakpoint $p^*$ satisfies:
\begin{equation} \label{eq:solve-optimal-break-point}
2mF(p^*) = 2m - p^* + m(m-2p^*)f(p^*). 
\end{equation}
By substituting~\eqref{eq:solve-optimal-break-point} in~\eqref{eq:pw-variance2} and simplifying, we obtain:
\begin{equation} \label{eq:pw-err-b}
\mathbb{E} (\varepsilon_{pw}^2; b, m, p^*) = C(b-1) \Big[ -(p^*)^2   + mp^* - m (m-2p^*)^2  f(p^*) \Big]. 
\end{equation}
Subtract the above from $\frac{C(b-1)}{16 C(b)}$ of \eqref{eq:uni-quant-err-[-m,m]}, we complete the proof:
\begin{equation} \label{eq:pw-err-compare}
\begin{array}{ll}
& \mathbb{E} (\varepsilon_{pw}^2; b, m, p^*) - \frac{C(b-1)}{16 C(b)} \mathbb{E} (\varepsilon_{uni}^2; b, -m, m) \vspace{1mm}\\
& = \mathbb{E} (\varepsilon_{pw}^2; b, m, p^*) - C(b-1) ( \frac{1}{4}m^2 ) \vspace{1mm} \\
& \leq C(b-1) \Big[ -(p^*-\frac{m}{2})^2- m (m-2p^*)^2  f(p^*) \Big] < 0.
\end{array}
\end{equation}
Note that $C(b) = \frac{1}{12 (2^b-1)^2}$ given from equation \eqref{eq:uni-quant-err}, for $b \geq 2$,
 \begin{equation}
      \frac{C(b-1)}{16 C(b)} = \frac{1}{16} \left(\frac{2^b-1}{2^{b-1}-1}\right)^2 = \frac{1}{16} \left(2 + \frac{1}{2^{b-1} - 1}\right)^{2} \leq \frac{9}{16}.
 \end{equation}
Therefore, $b$-bit PWLQ achieves a smaller quantization error, which is at most $\frac{9}{16}$ of $b$-bit uniform scheme. This improvement in performance requires only an extra bit for storage and no extra multiplication, as we discuss in the next section.
\end{proof}
\section{Hardware Impact}\label{S:HW_impact}

In this section, we discuss the hardware requirements for efficient deployment of DNNs quantized with PWLQ. In convolutional and fully-connected layers, every output can be computed using an inner product between vector $X$ and vector $W$, which correspond to the input activation and weight (sub)tensors respectively.

From scheme \eqref{eq:uniform-quant-scheme}, the approximated versions of uniform  quantization are $\hat{X} = s_x X_q + z_x I$ and $\hat{W} = s_w W_q $ (assuming symmetric quantization for weights), where $X_q$ and $W_q$ are quantized integer vectors from $X$ and $W$, $I$ is an identity vector, $s_x$, $s_w$ and $z_x$ are associated constant-valued scaling factors and offset, respectively. The output of this uniform quantization is:
\begin{equation} \label{eq:uni-comp}
\begin{array}{ll}
    \langle \hat{X}, \: \hat{W} \rangle &=  \langle s_x X_q + z_x I, \: s_w W_q  \rangle = C_0 \langle X_q, \: W_q \rangle  + C_1,
\end{array}
\end{equation}
where $\langle \cdot, \: \cdot \rangle$ is defined as vector inner product, $C_0 = s_x s_w$ and $C_1 = z_x s_w \langle W_q, \: I \rangle$ denote floating-point constant terms that can be pre-computed offline. 

Equation~\eqref{eq:uni-comp} implies that a uniformly quantized DNN requires two steps: (i) an integer-arithmetic (INT) inner product; and (ii) followed by a floating-point (FP) affine map. The expensive $O(\abs{W})$ (the size of vector $W$) FP operations $\langle \hat{X}, \: \hat{W} \rangle$ are then accelerated via INT operations $\langle X_q, \: W_q \rangle$, plus $O(1)$ FP re-scaling and adding operands using $C_0$ and $C_1$.

As we showed in Section~\ref{S:PW_def} when applying PWLQ on weights with one breakpoint, the algorithm breaks the ranges into non-overlapping regions ($R_1$ and $R_2$), which requires separate computational paths ($P_1$ and $P_2$) as each region has a different scaling factor. We set offsets $z_{w_1} = 0, z_{w_2} = p$ and denote scaling factors by $s_{w_1}, s_{w_2}$ in $R_1, R_2$, respectively. We also define by $\langle \cdot, \: \cdot \rangle _{R_i}$ the associated partial vector inner product, and $W_{q_i}$ the associated quantized integer vector of $W$ in region $R_i$ for $i=1,2$. Then $P_1$ is computed using the following equation:
\begin{equation} \label{eq:pw-comp-p1}
\begin{array}{ll}
    P_1 & = \langle s_x X_q + z_x I, \: s_{w_1} W_{q_1} \rangle_{R_1} 
     = C_2 \langle X_q, \: W_{q_1} \rangle_{R_1} + C_3.
\end{array}
\end{equation}
$P_2$ has additional terms as it has a non-zero offset $p$:
\begin{equation} \label{eq:pw-comp-p2}
\begin{array}{ll}
    P_2 & = \langle s_x X_q + z_x I, \: s_{w_2} W_{q_2} + p I \rangle_{R_2}  \vspace{1mm} \\
        & = C_4 \langle X_q, \: W_{q_2} \rangle_{R_2}  + C_5 \langle X_q, \: I \rangle_{R_2}  + C_6  ,
\end{array}
\end{equation}
 where $C_2$, $C_3$, $C_4$, $C_5$, and $C_6$  are constant terms, which can be pre-computed similar to $C_0$ and $C_1$ in \eqref{eq:uni-comp}.

As indicated by \eqref{eq:pw-comp-p1} and \eqref{eq:pw-comp-p2} for PWLQ compared to uniform quantization \eqref{eq:uni-comp}, the extra term $\langle X_q, \: I \rangle_{R_2}$ is needed due to the non-zero offset $p$, which sums up the activations corresponding to weights in $R_2$. Since most of the weights\footnote{Around 90\% of the weights are locating in the center region $R_1$ in our experiments.} are in $R_1$, these extra computations in $R_2$ rarely happen. In addition, FP re-scaling and adding are needed in each region, which also increases the overall FP operation overhead.

In short, an efficient hardware implementation of PWLQ requires: 
\begin{itemize}
    \item One multiplier for products in both of $\langle X_q, W_{q_1} \: \rangle_{R_1}$ and $\langle X_q, \: W_{q_2} \rangle_{R_2}$. \vspace{1mm}
    
    \item Three accumulators: one of each for sum of products in $P_1$ and $P_2$, and another one for activations in $P_2$. \vspace{1mm}
    
    \item At most one extra bit for storage\footnote{This extra storage cost can be further compressed by exploiting the non-uniform distribution of values \cite{bakunas2017efficient,park2018value}.}
    per weight value to indicate the region. Note that this extra bit does not increase the multiply-accumulate (MAC) computation and it is only used to determine the appropriate accumulator, which can be done in hardware at negligible cost on the MAC unit.
\end{itemize}

Based on the above explanation, it is clear that more breakpoints require more accumulators and more storage bits per weight tensor. Also, applying PWLQ on both weights and activations\footnote{Applying PWLQ on both weights and activations is discussed in the supplementary material.} requires accumulators for each combination of activation regions and weight regions, which translates to more hardware overhead. As a result, more than one breakpoint on the weight tensor or applying PWLQ on both weights and activations might not be feasible, from a hardware implementation perspective.
\section{Experiments}\label{S:experiments}

We evaluate the robustness of our proposed PWLQ scheme for post-training quantization on popular networks of several computer vision benchmarks: ImageNet classification \cite{russakovsky2015imagenet}, semantic segmentation and object detection on the Pascal VOC challenge \cite{everingham2010pascal}. 
In all experiments, we apply batch normalization folding \cite{jacob2018quantization} before quantization. 
For activations, we follow the profiling strategy in \cite{zhao2019improving} to sample from 512 training images, and collect the median\footnote{We test with the \textit{top-k} median and percentile-based \cite{li2019fully} approaches  and use the top-10 median method for better robustness of low-bit quantization. We refer to the supplementary material for details.} of the top-10 smallest and top-10 largest activation values for the minimum and maximum range boundaries at each layer, respectively.
During inference, we apply quantization after clipping with these ranges. Unless stated otherwise, we quantize all network weights \textit{per-channel} into 3-to-8 bits; and \textit{uniformly} quantize activations as well as pooling layers \textit{per-layer} into 8-bit. We perform all experiments in Pytorch 1.2.0 \cite{paszke2017automatic}.

\subsection{Ablation Study on ImageNet} \label{sec:ablation-study}

In this section, we conduct experiments on the ImageNet classification challenge \cite{russakovsky2015imagenet} and investigate the effectiveness of our proposed PWLQ method. We evaluate the top-1 accuracy performance on the validation dataset for three popular network architectures: Inception-v3 \cite{szegedy2016rethinking}, ResNet-50 \cite{he2016deep} and MobileNet-v2 \cite{sandler2018mobilenetv2}. We use torchvision\footnote{\url{https://pytorch.org/docs/stable/torchvision}} 0.4.0 and its pre-trained models for our experiments.

\subsubsection{Optimal Breakpoint Selection.} \label{sec:numerical-optimal-bkp}

In order to apply PWLQ, we first need to find the \textit{optimal breakpoints} to divide the quantization ranges into \textit{non-overlapping regions}. As stated in Section \ref{sec:error-analysis}, we assume weights and activations satisfy Gaussian or Laplacian distributions, then we find the optimal breakpoints by solving the optimization problem \eqref{eq:p_optimize}.

For the case of one optimal breakpoint $p^*$, we can iteratively find it by gradient descent since \eqref{eq:p_optimize} is convex; or using a simple and fast approximation of $p^*/m = \ln(0.8614 m + 0.6079)$ for normalized Gaussian. Experimental results show that the approximation obtains almost the same accuracy compared to gradient descent, while also being considerably faster. Therefore, unless stated otherwise we use this approximated version of the optimal breakpoint for the rest of this paper. We report results with other assumptions such as Laplacian distributions in the supplementary material.

Other works treat the data distributions differently: BiScaled-DNN~\cite{jain2019biscaled} proposes a ratio heuristic to divide the data into two overlapping regions; and V-Quant~\cite{park2018value} introduces a value-aware method to split them into two non-overlapping regions, e.g, 2\% (98\%) of large (small) values located in the tail (center) region, respectively. 
Our implementation results in Figure \ref{fig:breakpoint-and-non-overlap} (left) show that PWLQ with non-overlapping regions achieves a superior performance on low-bit quantization compared to BiScaled-DNN improved version\footnote{We improved the original BiScaled-DNN \cite{jain2019biscaled} by applying affine-based uniform scheme \eqref{eq:uniform-quant-scheme} on each region and per-channel quantization.} (denoted by BSD+) and V-Quant, especially with a large margin on 4-bit MobileNet-v2. Non-overlapping approach shortens the quantization ranges ($\Delta$ in \eqref{eq:uni-quant-err}) for the tail regions by $1.25\times$ to $2\times$. 
Therefore, both our choices of \textit{non-overlapping regions} and \textit{optimal breakpoints} have a significant impact on reducing the quantization error and improving the performance of low-bit quantized models.

\begin{figure}
\begin{center}
    \includegraphics[width=0.9\linewidth]{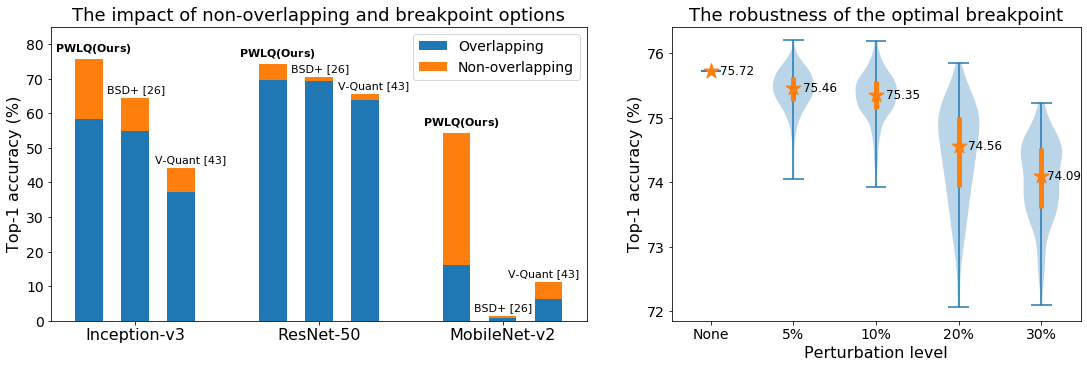} 
\end{center}
   \caption{Left: the impact of non-overlapping and breakpoint options on the top-1 accuracy for 4-bit post-training quantization models. Right: the robustness of the optimal breakpoint found by solving \eqref{eq:p_optimize} with some perturbation levels from 5\% to 30\% for 4-bit Inception-v3 (full-precision accuracy 77.49\%). Each perturbation level is run with 100 random samples, the star and the associated number indicate the median accuracy, the bold bar displays the accuracy range between the 25th and 75th percentiles  
}
\label{fig:breakpoint-and-non-overlap}
\end{figure}

In Figure \ref{fig:breakpoint-and-non-overlap} (right), we explore the robustness of the optimal breakpoint found by minimizing the quantization error in \eqref{eq:p_optimize} for 4-bit Inception-v3. We randomly add perturbation levels from 5\% to 30\% on each optimal breakpoint $p^*$ per-channel per-layer, e.g., the new breakpoint $\widehat{p^*} = 0.95 p^*$ or $1.05 p^*$ for 5\% of perturbation. We run 100 random samples for each perturbation level to generate the results. 
Overall, model performance decreases as the perturbation level increases, which indicates that our selection of the optimal breakpoint is crucial for accurate post-training quantization. Note that when 5\% of perturbation is added to our selection of optimal breakpoints, more than half of the experiments produce a lower accuracy, and can be as low as 74.05\%, which is a 1.67\% drop from the zero-perturbation baseline.

\subsubsection{Multiple Breakpoints.}

In this section, we discuss the trade-off of multiple breakpoints on model accuracy and hardware overhead. Theoretically, as the number of breakpoints on weights increases, the associated hardware cost linearly rises. Meanwhile, the number of non-overlapping regions and the associated total number of quantization levels grows, indicating a stronger representational power. Numerically, the extension of finding the optimal multi-breakpoints is straightforward by calculating the same quantization error \eqref{eq:pw-variance}, and solving the same optimization problem \eqref{eq:p_optimize} with gradient descent in an enlarged search space. 
Table \ref{tab:multiple-breakpoints} shows the accuracy performance up to three breakpoints. In general, using more breakpoints consistently improves model accuracy under the growing support of customized hardware. We suggest using one breakpoint to maintain the simplicity of the inference algorithm and its hardware implementation. Thus we only report PWLQ with one breakpoint for the rest of this paper.

 \begin{table}
\caption{Top-1 accuracy (\%) and requirement of hardware accumulators for PWLQ with multiple breakpoints on weights} 
\centering 
\resizebox{\columnwidth}{!}
{ 
{
\renewcommand{\arraystretch}{1.25}%
\begin{tabular}{c|c|ccc|ccc|ccc}
\toprule
\multirow{2}{*}{\begin{tabular}[c]{@{}c@{}}Number of\\\:Breakpoints\:\end{tabular}} &
\multirow{2}{*}{\begin{tabular}[c]{@{}c@{}}Hardware\\\:Accumulators\:\end{tabular}} &
\multicolumn{3}{c|}{Inception-v3 (77.49)} & \multicolumn{3}{c|}{ResNet-50 (76.13)} & \multicolumn{3}{c}{MobileNet-v2 (71.88)}\\ 
& & 5-bit & 4-bit & 3-bit & 5-bit & 4-bit & 3-bit & 5-bit & 4-bit & 3-bit   \\ 
\midrule
One   & \textbf{Three} & \:\:77.28\:\: & \:\:75.72\:\: & \:\:61.76\:\: & \:\:75.62\:\: & \:\:74.28\:\: & \:\:67.30\:\: & \:\:69.05\:\: & \:\:54.34\:\: & \:\:16.77\:\:  \\ 
Two   & Five & 77.31 & 76.73 & 71.40 & 75.94 & 75.24 & 73.27 & 70.01 & 65.74 & 36.44  \\ 
Three & Seven & \textbf{77.46} & \textbf{77.00} & \textbf{74.07} & \textbf{76.06} & \textbf{75.77} & \textbf{73.84} & \textbf{70.43} & \textbf{67.71} & \textbf{55.17}  \\ 
\bottomrule
\end{tabular}
}
}
\label{tab:multiple-breakpoints}
\end{table}

\subsubsection{PWLQ and Uniform Quantization.}

In Section \ref{sec:error-analysis}, we analytically and numerically demonstrate that our method, PWLQ, obtains a smaller quantization error than uniform quantization. We compare these two schemes in Table \ref{tab:pw-vs-uniform}. In this table, weights are quantized per-channel with the same computational bit-width $b=4, 6, 8$; activations are uniformly quantized per-layer into 8-bit. Generally, PWLQ achieves higher accuracy than uniform quantization except for one minor case of 8-bit Inception-v3. When the bit-width is large enough ($b=8$), the quantization error is small and both uniform quantization and PWLQ provide good accuracy. However, when the bit-width is decreased to 4, PWLQ obtains a notably higher accuracy, i.e., PWLQ attains 75.72\% but uniform quantization only attains 44.28\% for 4-bit Inception-v3. These results show that PWLQ is a more powerful representation scheme in terms of both quantization error and model accuracy, making it a viable alternative for uniform quantization in low bit-width cases. Moreover, PWLQ applies uniform quantization on each piece, hence it features a simple computational scheme and can benefit from any tricks that improve uniform quantization performance such as bias correction.

\begin{table}
\centering 
\caption{Comparison results of top-1 accuracy (\%) for uniform and PWLQ schemes on weights. $b$+BC: $b$-bit with bias correction for bit-width $b=4,6,8$. Each bold value indicates the best result from different methods for specified bit-width and network} 
\resizebox{\columnwidth}{!}
{
{
\renewcommand{\arraystretch}{1.25}%
\begin{tabular}{c|c|cc|cc|cc}
\toprule
Network                   
& Weight Bit-width & 8-bit & 8+BC & 6-bit & 6+BC & 4-bit & 4+BC \\
\midrule
\multirow{2}{*}{\begin{tabular}[c]{@{}c@{}}\:\:Inception-v3\:\:\\ (77.49)\end{tabular}} 
& Uniform    & 77.53  & 77.52  & 76.87  & 77.24  & 44.28  & 62.46 \\ 
&\:\: \textbf{PWLQ (Ours)}\:\:    & \:\:77.52\:\:  & \:\:\textbf{77.53}\:\: & \:\:77.42\:\: & \:\:\textbf{77.48}\:\: & \:\:75.72\:\: &  \:\:\textbf{76.45}\:\:   \\ 
\midrule 
\multirow{2}{*}{\begin{tabular}[c]{@{}c@{}}ResNet-50\\ (76.13)\end{tabular}}   
& Uniform    & 76.10  & \textbf{76.14} & 75.61  & 75.92  & 65.48 & 72.45\\ 
& \textbf{PWLQ (Ours)}    & 76.10 & 76.10   & 76.03   & \textbf{76.08} & 74.28 & \textbf{75.62} \\
\midrule
\multirow{2}{*}{\begin{tabular}[c]{@{}c@{}}MobileNet-v2\\ (71.88)\end{tabular}}    
& Uniform    & 71.35  & 71.58  & 67.76  & 70.81   & 11.37 & 41.80\\ 
& \textbf{PWLQ (Ours)}    & 71.59  & \textbf{71.73}  & 70.82  & \textbf{71.58}   & 54.34  & \textbf{69.22}           \\ 
\bottomrule
\end{tabular}
}
}
\label{tab:pw-vs-uniform}
\end{table}

\subsubsection{Bias Correction.}
An inherent bias in the mean and variance of the tensor values was observed after the quantization process and the benefits of correcting this bias term have been demonstrated in \cite{banner2018post,finkelstein2019fighting,nagel2019data}. This bias can be compensated by folding certain correction terms into the scale and the offset \cite{banner2018post}. We adopt this idea into our PWLQ method and show the results in Table \ref{tab:pw-vs-uniform} (columns with ``+BC"). Applying bias correction further improves the performance of low-bit quantized models. It allows 6-bit post-training quantization with piecewise linear scheme for all three networks to achieve near full-precision accuracy within a drop of 0.30\%; 4-bit MobileNet-v2, also without retraining, achieves an accuracy of 69.22\%. In general, a combination of low-bit PWLQ and bias correction on weights achieves minimal loss of full-precision model performance.

\subsection{Comparison to Existing Approaches}

In this section, we compare our PWLQ method with other existing approaches, by quoting the reported performance scores from the original literature.

An inclusive evaluation of clipping techniques along with outlier channel splitting (OCS) was presented in \cite{zhao2019improving}. To fairly compare with these methods, we adopt the same setup of applying per-layer quantization on weights and without quantizing the first layer. In Table \ref{tab:compare-pw-vs-clip-ocs}, we show that our PWLQ (no bias correction) outperforms the best results of clipping method combined with OCS. Besides, OCS needs to change the network architecture, in contrast to PWLQ.

\begin{table}
\centering
\caption{Comparison results of per-layer PWLQ and best clipping with OCS \cite{zhao2019improving} on top-1 accuracy (\%) loss. W/A indicate the bit-width on weights/activations. The accuracy difference values are measured from the full-precision (32/32) result}
\resizebox{\textwidth}{!}
{ 
\renewcommand{\arraystretch}{1.25}%
\begin{tabular}{c|ccccccc}
\toprule
Network                   
& W/A & 32/32 & 8/8 & 7/8 & 6/8 & 5/8 & 4/8 \\
\midrule
\multirow{2}{*}{\begin{tabular}[c]{@{}c@{}}Inception-v3\:\end{tabular}} 
& \:OCS + Best Clip\: & 75.9 & -0.6 (75.3)  & -1.2 (74.7) & -3.4 (72.5)  & -13.0 (62.9) & -71.1 (4.8)  \\ 
& \textbf{PWLQ (Ours)}   & 77.5 & \:\:\textbf{+0.1 (77.6)}\:\: & \:\:\textbf{-0.1 (77.4)}\:\: & \:\:\textbf{-0.3 (77.2)}\:\: & \:\:\textbf{-2.0 (75.5)}\:\: & \:\:\textbf{-12.8 (64.7)} \\
\midrule
\multirow{2}{*}{\begin{tabular}[c]{@{}c@{}}ResNet-50\end{tabular}}
& OCS + Best Clip & 76.1 & -0.4 (75.7) & -0.5 (75.6) & -0.9 (75.2)  & -2.7 (73.4) &  -6.8 (69.3)  \\ 
& \textbf{PWLQ (Ours)}   & 76.1 & \textbf{-0.0 (76.1)} & \textbf{-0.1 (76.0)} & \textbf{-0.2 (75.9)} & \textbf{-0.7 (75.5)} & \textbf{-2.4 (73.7)} \\
\bottomrule
\end{tabular}
}
\label{tab:compare-pw-vs-clip-ocs} 
\end{table}

In Table \ref{tab:compare-pw-with-all-others}, we provide a comprehensive comparison result of our PWLQ to other existing quantization methods. Here we apply per-layer quantization on activations and per-channel PWLQ on weights with bias correction. Except for the 4/4 case where we apply 4-bit PWLQ on activations, we always apply 8-bit uniform quantization on activations for the rest of the 8/8 and 4/8 cases. Under the same bit-width of computational cost among all the methods, our PWLQ combined with bias correction achieves the state-of-the-art results on all cases and it outperforms all other methods with a large margin on 4/8 and 4/4 cases. We emphasize that our PWLQ method is simple and efficient. It achieves the desired accuracy at the small cost of a few more accumulations per MAC unit and a minor overhead of storage. More importantly, it is orthogonal and applicable to other methods.

\begin{table}
\caption{Comparison of our PWLQ and other methods on top-1 accuracy (\%) loss. PWLQ: weights are piecewise linearly quantized per-channel with bias correction, activations are quantized per-layer}
\centering 
\resizebox{\textwidth}{!}
{
{
\renewcommand{\arraystretch}{1.25}%
\begin{tabular}{c|cccccccccc}
\toprule
Network                   
& \:\:W/A\:\:   & \textbf{PWLQ (Ours)} \:\:    & QWP \cite{krishnamoorthi2018quantizing}         & ACIQ \cite{banner2018post}        & LBQ  \cite{choukroun2019low}         & SSBD  \cite{meller2019same}      & QRD  \cite{lee2018quantization}  & UNIQ \cite{baskin2018uniq} & DFQ \cite{nagel2019data} \\ \midrule
\multirow{4}{*}{\begin{tabular}[c]{@{}c@{}}Inception-v3\\(Top1\%)\end{tabular}} 
& 32/32 & 77.49    & 78.00   & 77.20   & 76.23   & 77.90  & 77.97  & - & -  \\ \cmidrule(lr{1em}){2-10}
& 8/8   & \textbf{+0.04 (77.53) } & 0.00 (78.00)  & -   & -   & -0.03 (77.87)  & -0.09 (77.88)  & - & - \\
& 4/8   & \textbf{-1.04 (76.45)} & -7.00 (71.00) & -9.00 (68.20) & -1.44 (74.79) & -   & -  & - & - \\
& 4/4   & \textbf{-2.58 (74.91)} & -   & -10.80 (66.40) & -4.62 (71.61) & -     & -     & -     & -  \\
\midrule
\multirow{4}{*}{\begin{tabular}[c]{@{}c@{}}ResNet-50\\(Top1\%)\end{tabular}} 
& 32/32 &76.13 & 75.20 & 76.10 & 76.01 & 75.20  & - & 76.02 & - \\ \cmidrule(lr{1em}){2-10}
& 8/8   & \textbf{-0.03 (76.10)}  & -0.10 (75.10) & - & - & -0.25 (74.95) & - & - & - \\
& 4/8   & \textbf{-0.51 (75.62)} & -21.20 (54.00) & -0.80 (75.30) & -1.03 (74.98) & - & -  & -2.56 (73.37) & - \\
& 4/4   & \textbf{-1.28 (74.85) } & -  & -2.30 (73.80) & -3.41 (72.60) & - & - & - & - \\ \midrule
\multirow{3}{*}{\begin{tabular}[c]{@{}c@{}}MobileNet-v2\:\\(Top1\%)\end{tabular}} 
& 32/32 & 71.88 & 71.90 & - & - & 71.80 & 71.23  & - & 71.72 \\
\cmidrule(lr{1em}){2-10}
& 8/8   & \textbf{-0.15 (71.73)} & -2.10 (69.80) & - & - & -0.61 (71.19) & -1.68 (69.55)  & -  & -0.53 (71.19)\\
& 4/8   & \textbf{-2.68 (69.22)} & -71.80 (0.10) & - & - & - & - & - & - \\
\bottomrule
\end{tabular}
}
}
\label{tab:compare-pw-with-all-others}
\end{table}

\subsection{Other Applications}

To show the robustness and applicability of our proposed approach, we extend the PWLQ idea to other computer vision tasks including semantic segmentation on DeepLab-v3+ \cite{chen2018encoder} and object detection on SSD \cite{liu2016ssd}.

\subsubsection{Semantic Segmentation.}

In this section, we apply PWLQ on DeepLab-v3+ with a backbone of MobileNet-v2. The performance is evaluated using mean intersection over union (mIoU) on the Pascal VOC segmentation challenge \cite{everingham2010pascal}. 

In our experiments, we utilize the implementation of public Pytorch repository\footnote{\url{https://github.com/jfzhang95/pytorch-deeplab-xception}} to evaluate the performance. 
After folding batch normalization of the pre-trained model into the weights, we found that several layers of weight ranges become very large (e.g., [-54.4, 64.4]). Considering the fact that quantization range \cite{jung2019learning}, especially in the early layers \cite{choukroun2019low}, has a profound impact on the performance of quantized models, we fix the configuration of some early layers in the backbone. More precisely, we apply 8-bit PWLQ on three depth-wise convolution layers with large ranges in all configurations shown in Table \ref{tab:deep-lab}. Note that the MAC operations of these three layers are negligible in practice since they only contribute 0.2\% of the entire network computation, but it is remarkably beneficial to the performance of low-bit quantized models.

\begin{table}
\caption{Uniform quantization and PWLQ on DeepLab-v3+. Weights are quantized per-channel with bias correction, activations are uniformly quantized per-layer}
\centering 
\resizebox{\columnwidth}{!}
{ 
{
\renewcommand{\arraystretch}{1.25}%
\begin{tabular}{c|ccccc}
\toprule
Network                                                                                   
& W/A  & 32/32  & 8/8          & 6/8          & 4/8              \\ \midrule
\multirow{3}{*}{\begin{tabular}[c]{@{}c@{}}\:DeepLab-v3+\:\\ (mIoU\%)\end{tabular}} 
& Uniform   & 70.81  & -0.65 (70.16)  & -1.54 (69.27)  &     -20.76 (50.05) \\ 
& \:\:\textbf{PWLQ (Ours)}\:\:   & \:\:70.81\:\:  & \:\:\textbf{-0.12 (70.69)}\:\:    & \:\:\textbf{-0.42 (70.39)}\:\: & \:\:\textbf{-3.15 (67.66)}\:\: \\ 
& \:\:DFQ \cite{nagel2019data}\:\:   & 72.94  & -0.61 (72.33) & - & -  \\ \bottomrule
\end{tabular}
}
}
\label{tab:deep-lab}
\end{table}

As noticed in classification, low-bit uniform quantization causes significant accuracy drop from the full-precision models. In Table \ref{tab:deep-lab}, applying the piecewise linear method combined with bias correction, the 6-bit PWLQ model on weights even outperforms 8-bit DFQ \cite{nagel2019data}, which attains 0.42\% degradation of the pre-trained model. Moreover, the 4-bit PWLQ significantly improves the mIoU by 17.61\% from the 4-bit uniform quantized model, indicating the potential of low-bit post-training quantization via piecewise linear approximation for the semantic segmentation task.

\subsubsection{Object Detection.}

We also test the proposed PWLQ for the object detection task. The experiments are performed on the public Pytorch implementation\footnote{\url{https://github.com/qfgaohao/pytorch-ssd}} of SSD-Lite version \cite{liu2016ssd} with a backbone of MobileNet-v2. The performance is evaluated with mean average precision (mAP) on the Pascal VOC object detection challenge \cite{everingham2010pascal}. 

Table \ref{tab:ssd-lite} compares the results of the mAP score of quantized models using the uniform and PWLQ schemes. Similar to image classification and semantic segmentation tasks, even with bias correction and per-channel quantization enhancements, 4-bit uniform scheme causes 3.91\% performance drop from the full-precision model, while 4-bit PWLQ with these two enhancements is able to remove this notable gap down to 0.38\%.

\begin{table}
\caption{Uniform quantization and PWLQ of SSD-Lite version. Weights are quantized per-channel with bias correction, activations are uniformly quantized per-layer } 
\centering 
\resizebox{\columnwidth}{!}
{ 
{
\renewcommand{\arraystretch}{1.25}%
\begin{tabular}{c|ccccccc}
\toprule
Network                                                                                   
& W/A  & 32/32  & 8/8          & 6/8          & 4/8              \\ \midrule
\multirow{3}{*}{\begin{tabular}[c]{@{}c@{}}\:\:SSD-Lite\:\:\\ (mAP\%)\end{tabular}} 
& Uniform   & 68.70  & -0.20 (68.50)  & -0.43 (68.37)  & -3.91 (64.79)  \\ 
& \:\:\textbf{PWLQ (Ours)} \:\:   & \:\:68.70\:\:  &  \:\:\textbf{-0.19 (68.51)}\:\:  & \:\:\textbf{-0.28 (68.42)}\:\: & \:\:\textbf{-0.38 (68.32)}\:\: \\ 
& \:\:DFQ \cite{nagel2019data}\:\:   & 68.47 & -0.56 (67.91) & - & -  \\ \bottomrule
\end{tabular}
}
}
\label{tab:ssd-lite}
\end{table}

\section{Conclusion}\label{S:conclusion}

In this work, we present a piecewise linear quantization scheme for accurate post-training quantization of deep neural networks. It breaks the bell-shaped distributed values into non-overlapping regions per tensor where each region is assigned an equal number of quantization levels. We further analyze the resulting quantization error as well as the hardware requirements. We show that our approach achieves state-of-the-art low-bit post-training quantization performance on image classification, semantic segmentation, and object detection tasks under the same computational cost. It indicates its potential of efficient and rapid deployment of computer vision applications on resource-limited devices.


\bigskip
\noindent\textbf{Acknowledgements.}
We would like to thank Hui Chen and Jong Hoon Shin for valuable discussions.

\clearpage

%
%
\bibliographystyle{splncs04}
\bibliography{references}

\end{document}